\documentclass[preprint,12pt]{elsarticle}
\usepackage{nicefrac,amsmath,amssymb,inputenc,graphicx,xcolor,rotating,amsthm}
\usepackage{blindtext}
\usepackage{enumerate}
\usepackage{xcolor}
\usepackage{tcolorbox}
\usepackage{url}
\usepackage{hyperref}
\usepackage{subcaption}
\usepackage{caption}

\usepackage{tcolorbox}

\tcbset{
center,
top=5pt,
bottom=5pt,
width=22em,
colback=gray!20!,
colframe=pink!0!,
coltitle=black
}

\newtheorem{theorem}{Theorem}[section]
\newtheorem{corollary}[theorem]{Corollary}
\newtheorem{definition}[theorem]{Definition}
\newtheorem{proposition}[theorem]{Proposition}
\newtheorem{lemma}[theorem]{Lemma}

\newcommand{\NN}{\mathbb{N}}
\newcommand{\CC}{\mathbb{C}}
\newcommand{\ZZ}{\mathbb{Z}}

\newcommand{\inv}{^{-1}}

\makeatletter
\def\ps@pprintTitle{%
  \let\@oddhead\@empty
  \let\@evenhead\@empty
}
\makeatother

\journal{}

\begin{document}

\begin{frontmatter}


\title{Aliasing in Convnets: A Frame-Theoretic Perspective}

\author[ari]{Daniel Haider}
\author[ls2n]{Vincent Lostanlen}
\author[uniwien]{Martin Ehler}
\author[ari]{Nicki Holighaus}
\author[ari,itu]{and Peter Balazs}

\affiliation[ari]{
            organization={Acoustics Research Institute, Austrian Academy of Sciences},
            city={Vienna},
            postcode={1010},
            country={AUT}}
\affiliation[ls2n]{
            organization={Nantes Université, École Centrale Nantes, CNRS, LS2N},
            city={Nantes},
            postcode={44000}, 
            country={FRA}}
\affiliation[uniwien]{
            organization={Department of Mathematics, University of Vienna},
            city={Vienna},
            postcode={1090}, 
            country={AUT}}
\affiliation[itu]{
            organization={Interdisciplinary Transformation University Austria (IT:U)},
            city={Linz},
            postcode={4040}, 
            country={AUT}}

\begin{abstract}
    Using a stride in a convolutional layer inherently introduces aliasing, which has implications for numerical stability and statistical generalization. While techniques such as the parametrizations via paraunitary systems have been used to promote orthogonal convolution and thus ensure Parseval stability, a general analysis of aliasing and its effects on the stability has not been done in this context. In this article, we adapt a frame-theoretic approach to describe aliasing in convolutional layers with 1D kernels, leading to practical estimates for stability bounds and characterizations of Parseval stability, that are tailored to take short kernel sizes into account. From this, we derive two computationally very efficient optimization objectives that promote Parseval stability via systematically suppressing aliasing. Finally, for layers with random kernels, we derive closed-form expressions for the expected value and variance of the terms that describe the aliasing effects, revealing fundamental insights into the aliasing behavior at initialization.
\end{abstract}


\end{frontmatter}

\section{Introduction}
\label{sec:intro}
In signal processing, \textit{aliasing} describes the spectral overlap of different frequencies when sampling continuous signals, or applying a downsampling operation to a discrete signal. Fundamentally, aliasing relates to the fact that the samples of pure frequencies, i.e., linear-phase complex exponentials, are indistinguishable if they differ by a multiple of the sampling rate.
Besides analog-to-digital conversion, aliasing is a major concern in \textit{multi-rate signal processing}~\cite{vaidyanathan1993multi}, where convolution operators, usually termed as \textit{filters}, are composed with downsampling (and upsampling) operators to form \textit{multiple-input multiple-output (MIMO)} filterbanks. In the machine learning community, the composition of convolution and downsampling is widely referred to as \textit{strided convolution} and forms the fundamental building block of \emph{convolutional neural networks}: The linear part of any convolutional layer is functionally equivalent to a MIMO filterbank. 

Filterbanks, in particular \emph{invertible (or perfect reconstruction)} filterbanks have been intensively studied since the 1980s, with many authors contributing to a large corpus of literature. To date, Vaidyanathan's book~\cite{vaidyanathan1993multi} probably remains the most comprehensive reference. His
works~\cite{vaidyanathan1987pr,vaidyanathan1992pr}
and works of Vetterli et al.~\cite{119722,kovacevic1993perfect,vetterli1986filter,vetterli1989fir} may serve as representative examples of the crucial contributions of a large research community to our understanding of perfect reconstruction (PR) filterbanks. 
In most of these works, the so-called \emph{polyphase representation} was primarily used for the study of the PR property. In particular, the polyphase representation leads to a convenient characterization of critically decimated PR filterbanks as exactly those filterbanks with a polyphase representation given by a \emph{paraunitary matrix}, i.e., a matrix-valued function $R\colon \mathbb{C}\rightarrow \mathbb{C}^{M\times M}$ which satisfies $R(z)\overline{R^{\ast}(z^{-1})} =  I_M$ for all $z\in\mathbb{C}$. If the downsampling factor $d$ is smaller than the number of channels $M$, then $R(z)\in\mathbb{C}^{M\times d}$ is a matrix with orthonormal columns. Alternatively, PR filterbanks can be studied using their \emph{alias component matrix}, which is likewise a matrix-valued function $A\colon \mathbb{C}\rightarrow \mathbb{C}^{M\times d}$ each column of which is the frequency-domain form of one alias component of the filterbank. A classic equivalence between the polyphase and alias component matrix, see \cite[Eq.~(5.5.8)]{vaidyanathan1993multi}, shows that, in fact, $R(z)$ is paraunitary, if and only if $A(z)$ is paraunitary (if $d=M$), resp. orthonormal.

In the context of neural networks, these results have been used to characterize linear convolutional layers through paraunitary systems~\cite{su2022paraunitary, unser2024parseval}, and to establish a relation to the Cayley transform~\cite{trockman2021cayley}, both leading to efficient parametrization and design strategies.
However, a filterbank may be stable and invertible, even if it is not paraunitary. In that case, the alias component matrix can be used to explicitly describe the error introduced by aliasing by means of the derived \emph{Walnut representation}~\cite{balazs2017framespsycho}, thereby providing an intuitive description of a more general class of stable, invertible filterbanks. This is the perspective that we consider here, using tools (and terminology) from frame theory to describe and quantify aliasing in convolutional layers with 1D kernels~\cite{kiranyaz20211conv1d}.
The core idea is the following.
Let $\Theta$ be the operator that applies an undecimated filterbank and $\mathcal{F}$ the (unitary) Fourier transform. Then $\hat{S} = \mathcal{F}\Theta^* \Theta \mathcal{F}^*$ is a diagonal operator.
After introducing decimation, this is no longer true, and $\hat{S}$ contains additional non-zero side diagonals, which describe the effect of aliasing on the operator level by means of frequency correlation among the filters: $\hat{S}$ can be decomposed into a sum of (weighted) translation operators, the \textit{Walnut representation} of $\hat{S}$. First described by Walnut~\cite{walnut1992og} for Gabor systems this decomposition extends to arbitrary filterbanks with uniform~\cite{Janssen1998gaborbook} and non-uniform decimation~\cite{hernandez2002reproducing, holighaus2014nonstat}. It was further shown that quantifying the terms in the representation allows us to estimate stability bounds $A,B$ in the sense of
\begin{equation}\label{eq:stab}
 A \Vert x \Vert^2 \leq \big\Vert\Theta x \big\Vert^2 \leq B\Vert x \Vert^2
\end{equation}
for all inputs $x$. In fact, the existence of $0<A\leq B <\infty$ such that the above holds characterizes the invertibility and stability of the filterbank. In this case, we call it a \textit{frame}. Stability in the sense of \eqref{eq:stab} plays an important role in neural networks, where layer transforms with poor bounds $A,B$ can cause vanishing or exploding gradients, adversely affect training efficiency and generalization, and create vulnerabilities to adversarial attacks~\cite{glorot2010understand, santos2022overfitting, goodfellow2014adversarial}. To avoid instabilities, either a 1-Lipschitz constraint ($B\leq1$)~\cite{bethune2024lip, tsuzuku2018lip, pauli2022lip}, or even Parseval stability ($A=B=1$)~\cite{cisse2017parseval, brock2017orthoreg, hasannasab2020parseval} is imposed in the layers. In our work, we adapt the frame bound estimate from the $\ell^2(\ZZ)$ setting in~\cite{balazs2017framespsycho} to be applicable for convolutional neural networks with stride, and derive an extension that explicitly takes the kernel sizes of the convolutional layers into account. From our results we can derive two optimization objectives that promote tightness ($A=B$) by systematically suppressing aliasing. Directly formulated in the Fourier domain, they are computationally as cheap as the forward pass of a layer.
Overall, we provide a new perspective on the numerical stability of strided convolutional layers through aliasing, and indicate new approaches towards regularization via aliasing suppression.\\

The structure of the paper is as following.
In Section \ref{sec:basics} we prepare the overall setting for describing single-channel strided convolutional layers and their stability via filterbanks and g-frames in $\CC^L$. Section \ref{sec:walnut} introduces the Walnut representation for filterbanks in $\CC^L$ and presents estimates for stability bounds and characterizing conditions for Parseval stability, only based on the aliasing terms. We show that the conditions can be used as optimization objectives to systematically suppress aliasing and thereby promote Parseval stability as a direct consequence. Section \ref{sec:random} is dedicated to the expected value and variance of the aliasing terms for randomly initialized layers. In Section \ref{sec:multi}, we describe how all the results in the paper can be applied for multi-channel layers, non-uniform strides, and dilated convolution. Finally, Section~\ref{sec:conclusion} contains a concluding discussion, open problems, and an outlook.

\section{Filterbanks and Generalized Frames}\label{sec:basics}

\subsection{Basics and notation}
Throughout this paper, we consider finite complex-valued signals $x\in \CC^L$ and filters $w\in \CC^L$, both assumed to have periodic extension, i.e., $x [n+kL] = x [n]$ for any $k\in \mathbb{N}$.
This provides a setting where the convolution of a signal with a filter, given by
$(x  \ast w)[n] = \sum_{\ell=0}^{L-1}x [\ell]w[n-\ell]$ for $n=0,\ldots,L-1$ comes with circular boundary conditions. By time-reversing the filter, denoted by $(\mathbf{R}w)[n]=w[-n]$ we obtain cross-correlation, $(x \star w)[n] = (x \ast \mathbf{R}w)[n] = \sum_{\ell=0}^{L-1}x [\ell]w[n+\ell]$.
Let the (unitary) discrete Fourier transform (DFT) of length $L$ be given by
\begin{equation}
    \hat{x}[k] = (F x)[k] = \frac{1}{\sqrt{L}} \sum_{\ell = 0}^{L-1} x[\ell]e^{-2\pi ik\ell/L},
\end{equation}
where $F$ is the corresponding DFT matrix. From the circular boundary conditions we have that $\widehat{x\ast w} = \hat{x}\odot\hat{w}$,
where $\odot$ denotes the point-wise product of vectors. For a quadratic matrix $H\in \CC^{L\times L}$, we propose to call $\hat{H}=FHF^*$ the matrix Fourier transform of $H$.
We denote the translation by $a$ samples by $(\mathbf{T}_ax)[n]=x[n-a]$. For applying decimation on a vector by a factor $d\geq1$ we write $(x\hspace{-0.25em}\downarrow_{d})[n] = x[dn]$ and upsampling is denoted by $x\hspace{-0.25em}\uparrow_{d}[n] = x[n/d]$ if $n/d\in \mathbb{N}$ and $0$ otherwise.

\subsection{Strided convolution and filterbanks}

Applying convolution, followed by decimation is equivalent to apply convolution with a hop size, also known as \textit{strided convolution},
\begin{equation}\label{eq:conv}
    (x \ast w)\hspace{-0.25em}\downarrow_d[n] = \sum_{\ell=0}^{L-1}x [\ell]w[dn-\ell].
\end{equation}
In this context, the decimation factor $d$ is called the stride. We will assume that $L/d\in \mathbb{N}$ or, if necessary, extend the signal length to the next multiple of $d$, given by $d\lceil \tfrac{L}{d} \rceil$.

A linear single-channel \textit{strided convolutional layer} for $\CC^L$ is a \textit{filterbank} consisting of filters $(w_j)_{j=1}^M$ in $\CC^L$ that decompose a signal $x\in \CC^L$ into $M$ channels by strided convolution, i.e., for every $j=1,\dots,M$, the output is $y_j=(x \ast w_j)\hspace{-0.25em}\downarrow_d \in \CC^{L/d}$. We will use the notation
\begin{equation}
    \{(w_j)_{j=1}^M,\downarrow_d\}
\end{equation}
to denote such a filterbank. 
In practice, the filters in a convolutional layer have a customizable, preferably small, number of trainable parameters, much smaller than the signal length $L$ (common choices vary between 3 and 32~\cite{lou2019convtasnet,baevski2020wav2vec,chen2022wavlm}). We denote this number by $L_K$. In the filterbank paradigm this can be interpreted as the $w_j$ having non-zero entries only at the first $L_K\ll L$ coordinates. In this spirit, we will refer to the full-length vectors $w_j\in \CC^L$ as the \textit{filters} and to the non-zero part containing the $L_K$ (trainable) parameters as the \textit{kernels} of $w_j$.\\

To maintain clarity and conciseness, all results in this work will be formulated for the single-channel case. In Section \ref{sec:multi}, we discuss how they can be extended to multi-channel layers, non-uniform strides, and dilated convolution.

\subsection{Generalized frames}

To study the numerical stability of a filterbank we use a formalism that treats each of the strided convolutions as a separate linear operator. The corresponding notion is that of a generalized frame, or g-frame for short~\cite{sun2006gframes}.
\begin{definition}
    A g-frame is a collection $(T_j)_{j=1}^M$ of linear operators $T_j$ which satisfies that there are $A,B>0$ such that
\begin{equation}\label{eq:gframe}
    A\Vert x\Vert^2\leq \sum_{j=1}^M \Vert T_j x \Vert^2 \leq B\Vert x\Vert^2
\end{equation}
for all $x\in \CC^L$. The constants $A,B$ are called the \textit{frame bounds}. A g-frame is said to be \textit{tight} if $A=B$ and \textit{Parseval} if $A=B=1$. 
\end{definition}
For the transform $\Theta$, given by $x\mapsto (T_jx)_{j=1}^M$ (also called the analysis operator), the inequality in~\eqref{eq:gframe} represents a stability condition
which can be interpreted as relaxed variant of the energy preservation condition for Parseval stability of orthogonal operators. Moreover, $\Theta$ is injective if and only if $(T_j)_{j=1}^M$ is a g-frame.
The most important operator in this paper will be the \textit{frame operator}, defined by
\begin{align}
    S=\Theta^*\Theta:\CC^L &\rightarrow \CC^L\\
    x & \mapsto \sum_{j=1}^M T^*_j T_j x,
\end{align}
where $T^*_j$ is the adjoint operator of $T_j$ in the sense of $\langle T_jx,y \rangle_{\CC^{L'}} = \langle x,T_j^*y \rangle_{\CC^{L}}$ for $x\in \CC^L$ and $y\in \CC^{L'}$. In the case where $(T_j)_{j=1}^M$ is a g-frame for $\CC^L$, then $S$ is an invertible and self-adjoint $L\times L$ matrix whose smallest and largest eigenvalues are always positive and give the optimal values among all possible frame bounds $A,B$ in \eqref{eq:gframe}. Also in terms of tightness ($A=B$), the frame operator is of major importance.

\begin{lemma}
    The following are equivalent.
    \begin{enumerate}[(i)]
        \item $(T_j)_{j=1}^M$ is a tight g-frame for $\CC^L$.
        \item There is $A>0$ such that $\Vert \Theta x \Vert^2 
    = A\Vert x\Vert^2$ for all $x\in\CC^L$.
        \item There is $A>0$ such that $S=\hat{S}= A\cdot  I_L$.
    \end{enumerate}
\end{lemma}

For a filterbank $\{(w_j)_{j=1}^M,\downarrow_d\}$ in $\CC^L$,
the elements of the associated g-frame and their adjoints are given by
\begin{align}
\begin{split}
    T_j:\CC^L&\rightarrow \CC^{L/d}\hspace{2.5cm} T^*_j:\CC^{L/d}\rightarrow \CC^L\\
    x&\mapsto (x \ast w_j)\hspace{-0.25em}\downarrow_{d},\hspace{2.85cm} y\mapsto y\hspace{-0.25em}\uparrow_d\ast\; \overline{\mathbf{R}w_j},
\end{split}
\end{align}
and the frame operator applies as $Sx = \sum_{j=1}^M(((x\ast w_j)\hspace{-0.25em}\downarrow_d)\hspace{-0.25em}\uparrow_d)\ast \overline{\mathbf{R}w_j}$.
In many applications, using a filterbank that is a tight g-frame comes with many amenities. The first main advantage is that we can perfectly reconstruct any $x$ from its filterbank coefficients using the same filters. This is reflected by $\Theta^*\Theta x = Sx = Ax$. Secondly, energy preservation up to a constant in $(ii)$ provides robustness against perturbations~\cite{fickus2012robust}. In the neural network context, the Parseval case ($A=1$) is equivalent to what is referred to as \textit{orthogonal convolution}, and has been leveraged to stabilize the gradient flow, improve training efficiency and generalization, and increase the robustness against adversarial attacks.

\subsection{Extension to $\ell^2(\mathbb{Z})$}

For proving results that are independent of the signal length $L$, we will use that a filterbank in $\CC^L$ is completely determined by its extension to $\ell^2(\mathbb{Z})$. We use calligraphic symbols for all operators here.
Analogously to the definition of a g-frame for $\CC^L$, a filterbank
with filters $g_j \in \ell^2(\ZZ)$ and decimation factor $d$ is a g-frame for $\ell^2(\ZZ)$ if the frame inequalities from Eq.~\eqref{eq:gframe} hold with the corresponding strided convolution and norms in $\ell^2(\ZZ)$.
It is a tight g-frame if $A=B$, which happens if and only if the associated frame operator $\mathcal{S}:\ell^2(\ZZ)\rightarrow\ell^2(\ZZ)$ satisfies that $\mathcal{S}=A\cdot  I_{\ell^2(\ZZ)}$.
For $g\in \ell^2(\ZZ)$, let the discrete-time Fourier transform (DTFT) from $\ell^2(\ZZ)$ to $ L^2([0,1))$ be given by
\begin{equation}
    \hat{g}(\xi) = (\mathcal{F}{g})(\xi) = \sum_{\ell\in\ZZ} g[\ell]e^{-2\pi i\xi\ell}.
\end{equation}
Analogously, tightness is characterized by  $\hat{\mathcal{S}}=\mathcal{F}\mathcal{S}\mathcal{F}^*=A\cdot  I_{ L^2([0,1))}$.

\section{Stability Through Aliasing Suppression}\label{sec:walnut}

The aliasing effects in a filterbank that are introduced due to strided convolution can be expressed directly in terms of the matrix Fourier transform of the associated frame operator, $\hat{S}$. In particular, it can be written as a sum of terms that describe the frequency correlation among the filters. The so-called \textit{Walnut representation} of $\hat{S}$ makes this structure visible.
\subsection{Better call Walnut}
The following proposition is a finite dimensional version of Prop.~4 in \cite{balazs2017framespsycho}, which was originally formulated for filterbanks in $\ell^2(\ZZ)$.
\begin{proposition}[Walnut representation for filterbanks in $\CC^L$]\label{prop:walnut}
    \;\\Let $\{(w_j)_{j=1}^M,\downarrow_d\}$ be a filterbank in $\CC^L$ then
    \begin{equation}
        \hat{S}\hat{x} = \sum_{n=0}^{d-1} G_n \odot \mathbf{T}_{n\frac{L}{d}}\hat{x},
    \end{equation}
    where
    \begin{equation}
        G_{n} = d\inv\sum_{j=1}^{M} \hat{w}_j \odot \overline{\mathbf{T}_{n\frac{L}{d}}\hat{w}_j}
    \end{equation}
    are called the aliasing terms of the filterbank.
\end{proposition}
Although the proof is essentially the same as the one for~\cite[Prop.~4]{balazs2017framespsycho}, we include one using the terminology used in this paper.
\begin{proof}
For $x\in \CC^L$, we use the following two facts. First, the Fourier transform after down and upsampling is given by
\begin{equation}\label{eq:downup}
(F((x\hspace{-0.25em}\downarrow_d)\hspace{-0.25em}\uparrow_d))[k] = \sum_{n=0}^{L/d-1}x[nd]e^{-2\pi iknd/L}.
\end{equation}
Second, due to the periodic extension we can interpret $x$ as $L$-periodic sequence, and Poisson's summation formula gives for every $\ell=0,\dots, L-1$ that
\begin{equation}\label{eq:poisson}
\sum_{k=0}^{L/d - 1} x[kd] \, e^{-2\pi i \ell kd / L} = d\inv \sum_{n=0}^{d - 1} \hat{x}[\ell-n\tfrac{L}{d}].
\end{equation}
For every $\ell=0,\dots, L-1$, we derive
\begin{align}
    \hat{S}\hat{x}[\ell] = FS x[\ell]
    &= F\left(\sum_{j=1}^M(((x\ast w_j)\hspace{-0.25em}\downarrow_d)\hspace{-0.25em}\uparrow_d)\ast \overline{\mathbf{R}w_j}\right)[\ell] \\
    &= \sum_{j=1}^M F(((x\ast w_j)\hspace{-0.25em}\downarrow_d)\hspace{-0.25em}\uparrow_d)[\ell]\cdot F(\overline{\mathbf{R}w_j})[\ell]\\
    &= \sum_{j=1}^M\sum_{k=0}^{L/d-1} (x\ast w_j)[kd]e^{-2\pi i\ell kd/L}\cdot \overline{\hat{w}_j[\ell]}\hspace{1em}\text{by }\eqref{eq:downup}\\
    &= \sum_{j=1}^M\sum_{k=0}^{L/d-1} F^*(\hat{x}\odot \hat{w}_j)[kd]e^{-2\pi i\ell kd/L}\cdot \overline{\hat{w}_j[\ell]}\\
    &= \sum_{j=1}^{M}d\inv\sum_{n=0}^{d-1}\hat{x}[\ell-n\tfrac{L}{d}] \cdot \hat{w}_j[\ell-n\tfrac{L}{d}]\cdot \overline{\hat{w}_j[\ell]}\hspace{1em}\text{by }\eqref{eq:poisson}.
\end{align}
Exchanging the complex conjugations of the filters and isolating the aliasing terms yields the final expression.
\end{proof}
In other words, $\hat{S}$ is a band-diagonal matrix with $d$ bands that describe the frequency correlation among the filters. The entries of $\hat{S}$ are given by
\begin{equation}
    \hat{S}[k,\ell]=\begin{cases}
        G_{\lfloor \frac{\ell d}{L} \rfloor}[k\;\text{ mod } L/d]&\text{for }\; k-\ell\equiv 0\;\text{ mod } L/d\\
        0&\text{otherwise.}
    \end{cases}
\end{equation}

\begin{figure}[t]
    \centering
    \includegraphics[width=\linewidth]{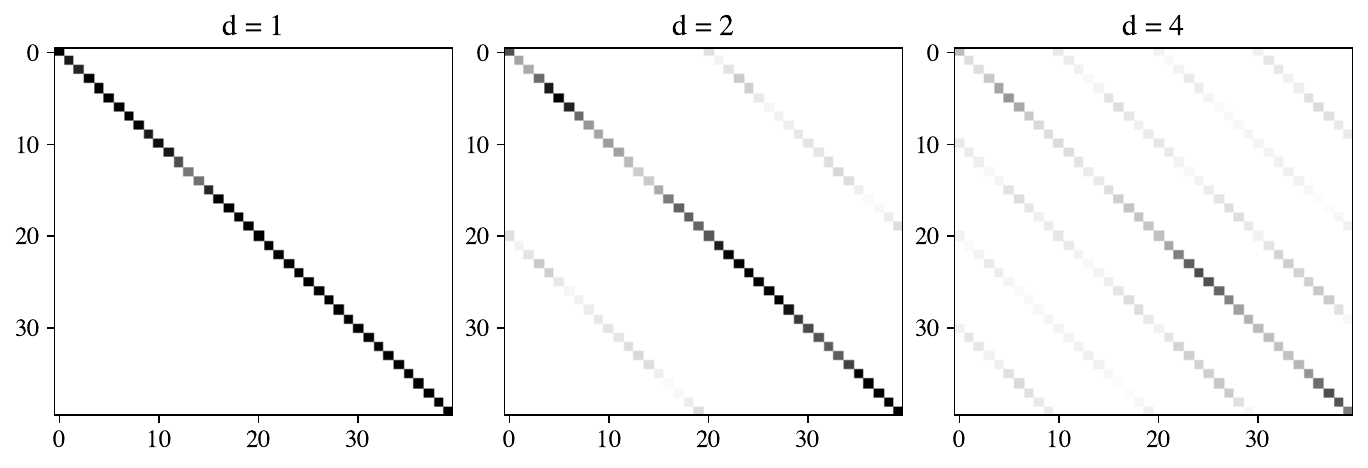}
    \caption{The matrix $\hat{S}=FSF^*$ for a filterbank for $\CC^{40}$ with $M=4$ i.i.d.~Gaussian complex random kernels of size $L_K=8$ and different strides $d=1,2,4$. 
    }
    \label{fig:shat}
    \includegraphics[width=\linewidth]{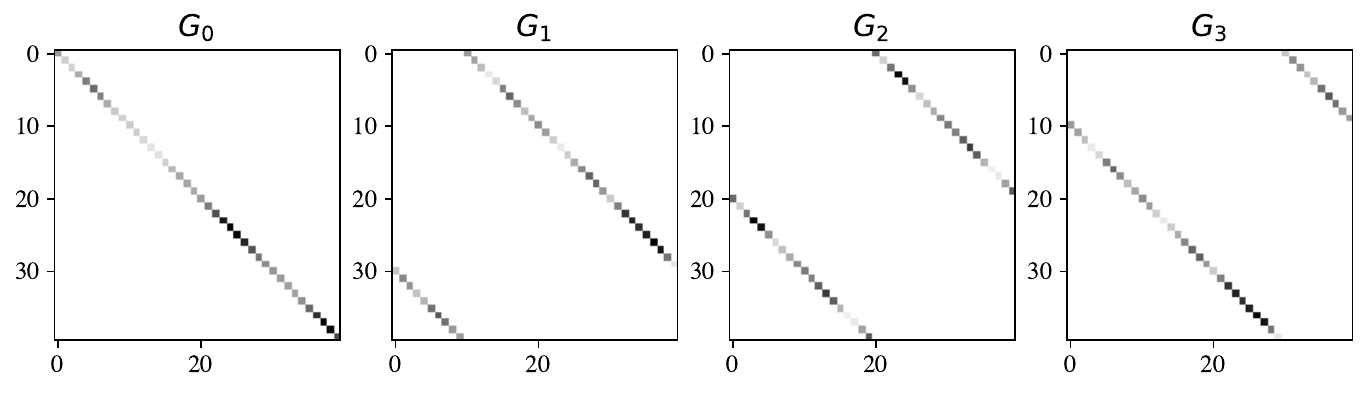}
    \caption{The aliasing terms $G_0,G_1,G_2,G_3$ isolated as the side diagonals of $\hat{S}$ for the right matrix ($d=4$) from above.
    }
    \label{fig:shat2}
\end{figure}

Figure \ref{fig:shat} and \ref{fig:shat2} illustrate the structure of $\hat{S}$ for a filterbank with random kernels and different decimation factors.
Setting $d=1$, we have that $\hat{S} = \operatorname{diag}\left( \sum_{j=1}^M \vert \widehat{w}_j \vert^2 \right) = \operatorname{diag}\left( G_0 \right)$.
In particular, $G_n\equiv 0$ for all $n>0$.
Hence, the stability of an undecimated filterbank is determined only by the filterbank response $G_0$. In particular, it is a tight g-frame if and only if the response is perfectly flat. The Walnut representation allows us to generalize these statements to $d>1$. By bounding the aliasing terms in different ways, we get frame bound estimates, and see that tightness is equivalent to a perfectly flat response, together with vanishing aliasing terms.

\subsection{Frame bound estimates and tightness characterizations}

For the estimates we are using a result that was originally formulated for filterbanks in $\ell^2(\mathbb{Z})$ in~\cite{balazs2017framespsycho}. In this setting, the aliasing terms are continuous functions on the torus, $\mathcal{G}_n\in L^2([0,1))$, given by
\begin{equation}
\mathcal{G}_n(\xi) = \sum_{j=0}^M d^{-1}\hat{g}_j(\xi)\overline{\hat{g}_j(\xi-nd^{-1})}.
\end{equation}
The frame bound estimates are obtained through diagonal dominance.
\begin{proposition}[G-frame formulation of \cite{balazs2017framespsycho}, Prop.~5]\label{lem:boundsl2}
    If there are $0<A\leq B<\infty$ with
    \begin{equation}
        A\leq \mathcal{G}_0(\xi) \pm \sum_{n=1}^{d-1} | \mathcal{G}_n(\xi)| \leq B
    \end{equation}
    for almost all $\xi\in [0,1)$, then the filterbank with filters $g_j \in \ell^2(\ZZ)$ and decimation factor $d$ is a g-frame for $\ell^2(\ZZ)$ with frame bounds $A,B$.
\end{proposition}
In the following theorem, we use Prop.~\ref{lem:boundsl2} to derive statements about the stability of a filterbank in $\CC^L$. We use the notation $\Vert x\Vert_\infty=\max_n |x[n]|$ and the fact that the average value of any $x\in \CC^L$ is given by $\frac{1}{\sqrt{L}}\hat{x}[0] = \frac{1}{L}\sum_{\ell=0}^{L-1}x[\ell]$.

\begin{theorem}\label{thm:main0}
    Let $\{(w_j)_{j=1}^M,\downarrow_d\}$ be a filterbank for $\CC^L$ and let $G_n$ be the associated aliasing terms. The following holds.
    \begin{enumerate}[(i)]
        \item Let
        \begin{align}
            A&= \min_k (G_0[k]-\sum_{n=1}^{d-1}|G_n[k]|)\\
            B&= \max_k (G_0[k]+\sum_{n=1}^{d-1}|G_n[k]|).
        \end{align}
        If $A>0$ then the filterbank is a g-frame for $\CC^{L}$ and $A,B$ are frame bounds.
        \item The filterbank is a tight g-frame for $\CC^{L}$ if and only if $G_0\equiv A$ for some $A>0$ and $G_n\equiv 0$ for every $n=1,\dots,d-1$. In particular, $A=\tfrac{1}{\sqrt{L}}\hat{G}_0[0]$.
        \item If the filterbank is a g-frame and $S$ is the frame operator then
        \begin{equation}
            \Vert S-  I_{L} \Vert \leq
            \Vert G_0-\mathbf{1}_L\Vert_\infty + \sum_{n=1}^{d-1} \Vert G_n \Vert_\infty,
        \end{equation}
        where $\mathbf{1}_L=(1,\dots,1)\in \CC^L$. Moreover, for $A,B$ as in $(ii)$ we have
        \begin{equation}
            \Vert S- \tfrac{1}{\sqrt{L}}\hat{G}_0[0]\cdot I_{L} \Vert\leq \max(|B-\tfrac{1}{\sqrt{L}}\hat{G}_0[0]|,|A-\tfrac{1}{\sqrt{L}}\hat{G}_0[0]|).
        \end{equation}
  \end{enumerate}    
\end{theorem}

\begin{proof}
    The frame bound estimates in $(i)$ follow directly from the $\ell^2(\ZZ)$ version in Prop.~\ref{lem:boundsl2}. Point $(ii)$ follows from the Walnut representation (Prop.~\ref{prop:walnut}) and that $\tfrac{1}{\sqrt{L}}\hat{G}_0[0]$ is the average value of $G_0$. For the first estimate in $(iii)$ we use the Walnut representation and the triangle inequality to get
    \begin{align}
    \begin{split}
        \Vert (S- I_L)x \Vert &= \Vert (\hat{S}- I_L)\hat{x} \Vert
        \leq \left(\Vert G_0-\mathbf{1}_L \Vert_\infty + \sum_{n=1}^{d-1} \Vert G_n \Vert_\infty\right)\Vert \hat{x} \Vert.
    \end{split}
    \end{align}
    For the second estimate we use that
    \begin{equation}\label{eq:star}
        \Vert S-\tfrac{1}{\sqrt{L}}\hat{G}_0[0]\cdot I_L\Vert = \max(|A^*-\tfrac{1}{\sqrt{L}}\hat{G}_0[0]|,|B^*-\tfrac{1}{\sqrt{L}}\hat{G}_0[0]|),
    \end{equation}
    where $A^*,B^*$ are the optimal frame bounds of $S$. Let $A,B$ be as in $(ii)$ then $A\leq A^*\leq B^* \leq B$ holds. Since $A\leq \min_k G_0[k]\leq\tfrac{1}{\sqrt{L}}\hat{G}_0[0]$, replacing $A^*,B^*$ by $A,B$ in~\eqref{eq:star} remains true, and the claim follows. 
\end{proof}

These results naturally relate aliasing to the stability of a filterbank by quantifying how reduced aliasing leads to improved frame bounds. In particular, they show that tightness is achieved precisely when aliasing is completely absent.
However, the results do not take the kernel sizes $L_K$ of the filters into account, which limits their practicality for long or varying signal lengths. In the following, we present a length-independent version of Theorem \ref{thm:main0}, tailored to the structural constraints of convolutional layers. The conditions that we derive are only required to hold for the minimal signal length of $L= d\lceil \tfrac{2L_K-1}{d} \rceil$, and maintain valid for all longer signal lengths.

For the proof, we will use the classic result that shifted Dirichlet kernels are an orthonormal basis for the space of trigonometric polynomials.

\begin{lemma}[\cite{vetterli2014foundationsofsignalprocessing}, Thm. 5.24]\label{lem:1}
    Let $p$
    be a trigonometric polynomial of degree $\leq N$ and $D_n(t) = \sum_{|\ell|\leq n} e^{2\pi i \ell t}$ denote the $n$-th Dirichlet kernel. For any $N'\geq N$, $p$ can be written as the semi-discrete convolution
        \begin{equation}
            p(t) = \tfrac{1}{\sqrt{2N'+1}}\sum_{n = 0}^{2N'} p\left(\tfrac{n}{2N'+1}\right)D_{N'}\left(t -\tfrac{n}{2N'+1}\right).
        \end{equation}
    Moreover, the Fourier coefficients of $p$ coincide with the DFT coefficients of the vector $\left(p(0),p\left(\tfrac{1}{2N'+1}\right),\ldots, p\left(\tfrac{2N'}{2N'+1}\right)\right)$.
\end{lemma}
The following theorem represents the main result of this paper.
\begin{theorem}\label{thm:main}
    Let $\{(w_j)_{j=1}^M,\downarrow_d\}$ with kernel size $L_K$ be a filterbank for $\CC^L$, where $L= d\lceil \tfrac{2L_K-1}{d} \rceil$. Let $G_n$ be the associated aliasing terms (of length $L$). For all $L'\geq L$ the following holds.
    \begin{enumerate}[(i)]
        \item
        Let
        \begin{align}
            A&= \tfrac{2}{\sqrt{L}}\hat{G}_0[0]-\sum_{n=0}^{d-1}\Vert \hat{G}_n \Vert_{1}\\
            B&= \sum_{n=0}^{d-1}\Vert \hat{G}_n \Vert_{1}.
        \end{align}
        If $A>0$ then the filterbank is a g-frame for $\CC^{L'}$ and $A,B$ are frame bounds.
        \item The filterbank is a tight g-frame for $\CC^{L'}$ if and only if $G_0\equiv A$ for some $A>0$ and $G_n\equiv 0$ for every $n=1,\dots,d-1$. In particular, $A=\tfrac{1}{\sqrt{L}}\hat{G}_0[0]$.
        \item If the filterbank is a g-frame for $\CC^{L}$ and $S$ the frame operator for $\CC^{L'}$ then
        \begin{equation}
            \Vert S- \tfrac{1}{\sqrt{L}}\hat{G}_0[0]\cdot  I_{L'} \Vert \leq \left\vert \tfrac{1}{\sqrt{L}}\hat{G}_0[0]-\sum_{n=0}^{d-1}\Vert \hat{G}_n \Vert_1 \right\vert.
        \end{equation}
  \end{enumerate}    
\end{theorem}

\begin{proof}
Let us assume that the filters $w_j$ have kernel size $L_K=2N+1$ for simplicity. Now let $g_j\in \ell^2(\ZZ)$ be the filters that coincide with the kernels of the $w_j$ on $[-N,N]$. The $\hat{g}_j$ are trigonometric polynomials of degree $\leq N$, hence, the summands in the associated aliasing terms $\mathcal{G}_n$ are all of the form 
$$
h_j(\xi) = \hat{g}_j(\xi)\cdot\overline{ \hat{g}_j(\xi - \omega)}, 
$$
with appropriate shifts $\omega$. Since the $h_j$ are trigonometric polynomials of degree $\leq 2N$, every $\mathcal{G}_n$ is likewise a trigonometric polynomial of degree $\leq 2N$.
The core argument for this proof is now provided by Lemma \ref{lem:1}, giving us for all $L'\geq 2N$ and $n=0,\ldots,d-1$, that
\begin{align}
    \mathcal{G}_n(\xi)
    = \tfrac{1}{\sqrt{2L'+1}}\sum_{\ell = 0}^{2L'} \mathcal{G}_n\left(\tfrac{\ell}{2L'+1}\right)D_{L'}\left(\xi -\tfrac{\ell}{2L'+1}\right),
\end{align}
and that the (non-zero) Fourier coefficients of $\mathcal{G}_n$ are given by $\hat{G}_n$. Assuming that $L_K=2N+1$, the minimum number of samples needed is $L=4N+1=2L_K-1$. Taking the divisibility condition for the decimation factors into account gives us $d\lceil \tfrac{2L_K-1}{d} \rceil$ as smallest admissible signal length. In other words, the non-zero Fourier coefficients of the aliasing terms for the $\ell^2(\mathbb{Z})$ extension of the filterbank, hence, also those for $\CC^{L'}$ where $L'\geq L$ are already completely determined by $\hat{G}_n$ (of length $L$).

To prove $(i)$, we approximate the estimates from Lemma \ref{lem:boundsl2} using that
\begin{equation}
\Vert\mathcal{G}_n\Vert_{L^\infty([0,1))}\leq\Vert\hat{\mathcal{G}}_n\Vert_{L^1([0,1))}=\Vert\hat{G}_n\Vert_{1}.
\end{equation}
For the upper bound, we obtain
\begin{align}
    \sup_{\xi\in [0,1)}&\left( \mathcal{G}_0(\xi) + \sum_{n=1}^{d-1}|\mathcal{G}_n(\xi)| \right)
    = \sup_{\xi\in [0,1)}\left( \sum_{n=0}^{d-1}|\mathcal{G}_n(\xi)| \right)\\
    &\leq \sum_{n=0}^{d-1}\Vert \mathcal{G}_n \Vert_{L^\infty([0,1))}
    \leq \sum_{n=0}^{d-1}\Vert \hat{G}_n \Vert_{1}.
\end{align}
For the lower bound, we use that
$\Vert \hat{G}_0 \Vert_1 - \tfrac{1}{\sqrt{L}}\hat{G}_0[0]$ bounds the maximal deviation in $\mathcal{G}_0$. We get
\begin{align}
    \inf_{\xi\in [0,1)}&\left( \mathcal{G}_0(\xi) - \sum_{n=1}^{d-1}|\mathcal{G}_n(\xi)| \right)
    \geq \inf_{\xi\in [0,1)}\left( \mathcal{G}_0(\xi) \right) - \sum_{n=1}^{d-1}\Vert \mathcal{G}_n \Vert_{L^\infty([0,1))}\\
    &\geq \tfrac{1}{\sqrt{L}}\hat{G}_0[0]- \left(\Vert \hat{G}_0 \Vert_1 - \tfrac{1}{\sqrt{L}}\hat{G}_0[0]\right) - \sum_{n=1}^{d-1}\Vert \hat{G}_n \Vert_{1}\\
    &= \tfrac{2}{\sqrt{L}}\hat{G}_0[0]-\sum_{n=0}^{d-1}\Vert \hat{G}_n \Vert_{1}.
\end{align}

To prove $(ii)$, we use that 
\begin{equation}\label{eq:bn}
\|\mathcal{G}_n\|_{L^2([0,1))} = \|\hat{G}_n\| = \|G_n\|,
\end{equation}
where the norms without subscript are Euclidean vector norms. Assuming that $\{(w_j)_{j=1}^M,\downarrow_d\}$ is tight on $\CC^{L'}$, then the diagonal entries of $\hat{S}$ are uniquely determined by the $L$ samples given by $G_0$, hence, they must all be equal. On the other hand, by \eqref{eq:bn}, all the $G_n$, $n> 0$ must vanish.

For the estimate in $(iii)$ we use that
\begin{equation}
    \Vert S-\tfrac{1}{\sqrt{L}}\hat{G}_0[0]\cdot I_{L'}) \Vert = \max(|A^*-\tfrac{1}{\sqrt{L}}\hat{G}_0[0]|,|B^*-\tfrac{1}{\sqrt{L}}\hat{G}_0[0]|),
\end{equation}
where $A^*,B^*$ are the optimal frame bounds of $S$. Let $A,B$ be as in $(i)$ then $A\leq A^*\leq B^* \leq B$. We obtain
\begin{align}
    \begin{split}
        \Vert S-\tfrac{1}{\sqrt{L}}\hat{G}_0[0]\cdot I_{L'}) \Vert &\leq \max\left(\tfrac{1}{\sqrt{L}}\hat{G}_0[0]-\sum_{n=0}^{d-1}\Vert \hat{G}_n \Vert_{1}, \sum_{n=0}^{d-1}\Vert \hat{G}_n \Vert_{1} -\tfrac{1}{\sqrt{L}}\hat{G}_0[0] \right) \\&= \left\vert \tfrac{1}{\sqrt{L}}\hat{G}_0[0] - \sum_{n=0}^{d-1}\Vert \hat{G}_n \Vert_1 \right\vert
    \end{split}
    \end{align}
\end{proof}

While the estimates in Thm.~\ref{thm:main} are looser than the ones in Thm.~\ref{thm:main0} due to the additional estimation step via the 1-norm, we have been able to detach the frame bounds from the signal length. As a consequence, all the conditions for tightness can be reduced to this minimal signal length of essentially twice the kernel size.

\begin{corollary}\label{cor:tight1} 
    Let $\{(w_j)_{j=1}^M,\downarrow_d\}$ with kernel size $L_K$ be a filterbank for $\CC^L$, where $L= d\lceil \tfrac{2L_K-1}{d} \rceil$. Let $G_n$ be the associated aliasing terms and $S$ the frame operator. The following are equivalent.
    \begin{enumerate}[(i)]
        \item The filterbank is tight on $\CC^{L'}$ for all $L'\geq L$.
        \item $S=\hat{S}=\tfrac{1}{\sqrt{L}}\hat{G}_0[0]\cdot I_L$
        \item $\frac{\max_k G_0[k ]}{\min_k G_0[k ]} + \sum_{n=1}^{d-1} \Vert G_n \Vert_\infty = 1$
        \item $\sum_{n=0}^{d-1}\Vert \hat{G}_n \Vert_1=\tfrac{1}{\sqrt{L}}\hat{G}_0[0].$
    \end{enumerate}
\end{corollary}
While $(ii)$ further specifies the standard tightness condition via the actual value of the bound, the conditions in $(iii)$ and $(iv)$ fully leverage the structure of $\hat{S}$ and only depend on the aliasing terms $G_n$. In particular, $(iii)$ concisely characterizes tightness by aliasing cancellation plus response equalization (note that also any other vector norm can be considered instead of $\Vert G_n \Vert_\infty$).
Condition $(iv)$ further specifies this on the Fourier level of the aliasing terms, thereby, elegantly describing aliasing cancellation in a simple formula. Note that since the $G_n$ are objects in the Fourier domain, $(iv)$ actually becomes a time-domain condition again. Moreover, since it already involves the tightness bound explicitly, we obtain Parseval stability for free.

\subsection{Painless hybrid filterbanks}

A further generalization of tightness is when $\hat{S}$ is diagonal matrix, but not necessarily with constant entries.
This situation is known as the \textit{painless} case~\cite{daubechies1986painless}, and stability obviously only depends on $G_0$ anymore, just as in the undecimated case. Classically, this is enforced by designing the filters in a way that they have limited band-pass, defined by the reciprocal decimation factor~\cite{balazs2011nonstat}. A simple example is to use ideal band-pass filters $\psi_j$ that satisfy 
\begin{equation}\label{eq:box}
    \hat{\psi}_j[k] = \begin{cases}
        1&\text{if } k\in [aj, aj+\tfrac{L}{d}-1]\\
        0&\text{else,}
    \end{cases}
\end{equation}
for some $a\leq \tfrac{L}{d}$. To make use of this in a strided convolutional layer we can combine the kernels with fixed ideal band-pass filters via pair-wise convolution. Let $(\psi_j)_{j=1}^M$ and $(w_j)_{j=1}^M$ be filters. The filterbank given by 
\begin{equation}
    \{(\psi_j \ast  w_j)_{j=1}^M,\downarrow_d\}.
\end{equation}
is called a \textit{hybrid filterbank}~\cite{haider2024holdmetight}. It is easy to see that a hybrid filterbank inherits the painless property from the ideal band-pass filters.

\begin{proposition}
    Let $\{(\psi_j \ast  w_j)_{j=1}^M,\downarrow_d\}$ be a hybrid filterbank for $\CC^L$ composed of ideal band-pass filters $(\psi_j)_{j=1}^M$ satisfying \eqref{eq:box} for $a=\tfrac{L}{d}$ and filters $(w_j)_{j=1}^M$. The hybrid filterbank is painless and
    \begin{enumerate}[(i)]
        \item a g-frame for $\CC^L$ if and only if $\min_k|\hat{w}_{\lfloor\frac{kd}{L}\rfloor} [k ]|>0$.
        \item a tight g-frame for $\CC^L$ if and only if $\vert\hat{w}_{\lfloor\frac{kd}{L}\rfloor} [k ]\vert=\vert\hat{w}_{\lfloor\frac{\ell d}{L}\rfloor} [\ell ]\vert\neq 0$ for all $k,\ell=0,\dots, L-1$.
    \end{enumerate}
\end{proposition} 

\begin{proof}
    The aliasing terms for a hybrid filterbank are given by
    \begin{equation}
        G_{n}^{(\psi,w)}[k] = \sum_{j=1}^{M} d\inv \hat{\psi}_j [k] \overline{\hat{\psi}_j [k-n\tfrac{L}{d}]}  \hat{w}_j [k] \overline{\hat{w}_j [k-n\tfrac{L}{d}]}.
    \end{equation}
    By \eqref{eq:box} we have for all $j=1,\dots, M$, $k=0,\dots, L-1$,
    \begin{equation}
        \hat{\psi}_j [k] \overline{\hat{\psi}_j [k-n\tfrac{L}{d}]} = 0.
    \end{equation}
    It follows that $G_{n}^{(\psi,w)}\equiv 0$ for $n>0$. In the diagonal term, only the non-overlapping squared frequency responses remain such that we get
    $$G_{0}^{(\psi,w)}[k] = d\inv \vert \hat{w}_{\lfloor\frac{kd}{L}\rfloor} [k ] \vert^2.$$
    The claims follow directly from this.
\end{proof}
An analog statement can be done for $a< \tfrac{L}{d}$, where the condition depends on the sum of all overlapping $\hat{w}_j$.
In other words, we split up the input signal into different frequency regions and process the resulting sub-band signals further, separately. This principle is very useful for audio processing and has been already leveraged for different architectures and applications~\cite{haider2024holdmetight, zeghidour2021leaf, ravanelli2018sincnet, lostanlen2023murenn}. In another context, the fixed decomposition into sub-bands has been shown to stabilize the layer in a random setting~\cite{lostanlen2025reshyb}, and to be advantageous for making deconvolution more robust~\cite{nenov2025snake}.

\subsection{Numerical experiments on suppressing aliasing}\label{sec:unlearn}

Parseval stability in neural networks has been used against exploding gradients and for enhancing the robustness against noise and adversarial attacks.
To maintain this during training, a sequence of works~\cite{brock2017orthoreg, cisse2017parseval,wang2020ocnn, massart2022rectorth} has proposed to minimize
\begin{equation}\label{eq:regS}
    \mathcal{L}^S = \Vert S -  I_L \Vert,
\end{equation}
(or some variant thereof) for all desired layers, alongside a learning objective. Other possibilities to maintain tightness are to minimize $\mathcal{L}^\kappa=\frac{B}{A}-1$~\cite{balazs2024stableencoders}
or $\mathcal{L}^\Theta=\frac{1}{2}(\Vert \Theta\Vert_2^2-\frac{1}{L}\Vert \Theta\Vert_F^2)$~\cite{nenov2024smoothsailing}.
By Corollary \ref{cor:tight1}, we can promote tightness also by suppressing aliasing via minimizing
\begin{equation}\label{eq:reg0}
    \mathcal{L}^G = \frac{\max_k G_0[k]}{\min_k G_0[k]} -1 +\sum_{n=1}^{d-1}\Vert G_n \Vert 
\end{equation}
or
\begin{equation}\label{eq:reg}
    \mathcal{L}^{\hat{G}} = \left|\tfrac{1}{\sqrt{L}}\hat{G}_0[0]-\sum_{n=0}^{d-1}\Vert \hat{G}_n \Vert_1\right|.
\end{equation}
The above expressions can be computed very efficiently using fast Fourier transforms. The computational complexity for $\mathcal{L}^{G}$ is $\mathcal{O}\left(M L \log L + M d L\right)$, and for $\mathcal{L}^{\hat{G}}$ it is $\mathcal{O}\left((M+d) L \log L + M d L\right)$. In comparison, for $\mathcal{L}^S$ it is $\mathcal{O}\left(\frac{M L^3}{d} + L^3\right)$.\\

As a demonstration, we optimize a linear convolutional layer with $M$ kernels of size $L_K$ and stride $d$ to become approximately Parseval stable for $\CC^L$, $L={d\lceil \tfrac{2L_K-1}{d} \rceil}$ by minimizing the proposed objectives $\mathcal{L}^G$ \eqref{eq:reg0} and $\mathcal{L}^{\hat{G}}$ \eqref{eq:reg}, and compare it to the baseline $\mathcal{L}^S$ \eqref{eq:regS}. We do not consider $\mathcal{L}^\kappa$ and $\mathcal{L}^\Theta$ here since they yield similar results but are very slow. However, we consider a brute-force method based on the construction of the canonical Parseval g-frame as a reference.
Given $(T_j)_{j=1}^M$, then $(T_jS^{-\frac{1}{2}})_{j=1}^M$ is the closest Parseval g-frame in the Frobenius norm among all Parseval g-frames~\cite{ehler2015pre}. Note that this corresponds to the filterbank $\{(S^{-\frac{1}{2}}w_j)_{j=1}^M,\downarrow_d\}$. Since applying $S^{-\frac{1}{2}}$ does not maintain the kernel size, we project the filters back onto the first $L_K$ components and repeat the process. This idea is similar to~\cite[Sec.~8.3]{strohmer1998gaborbook}. We call this FIR-tightening.

\begin{figure}[ht!]
    \centering
    \begin{subfigure}[t]{0.49\textwidth}
        \centering
        \includegraphics[width=\textwidth]{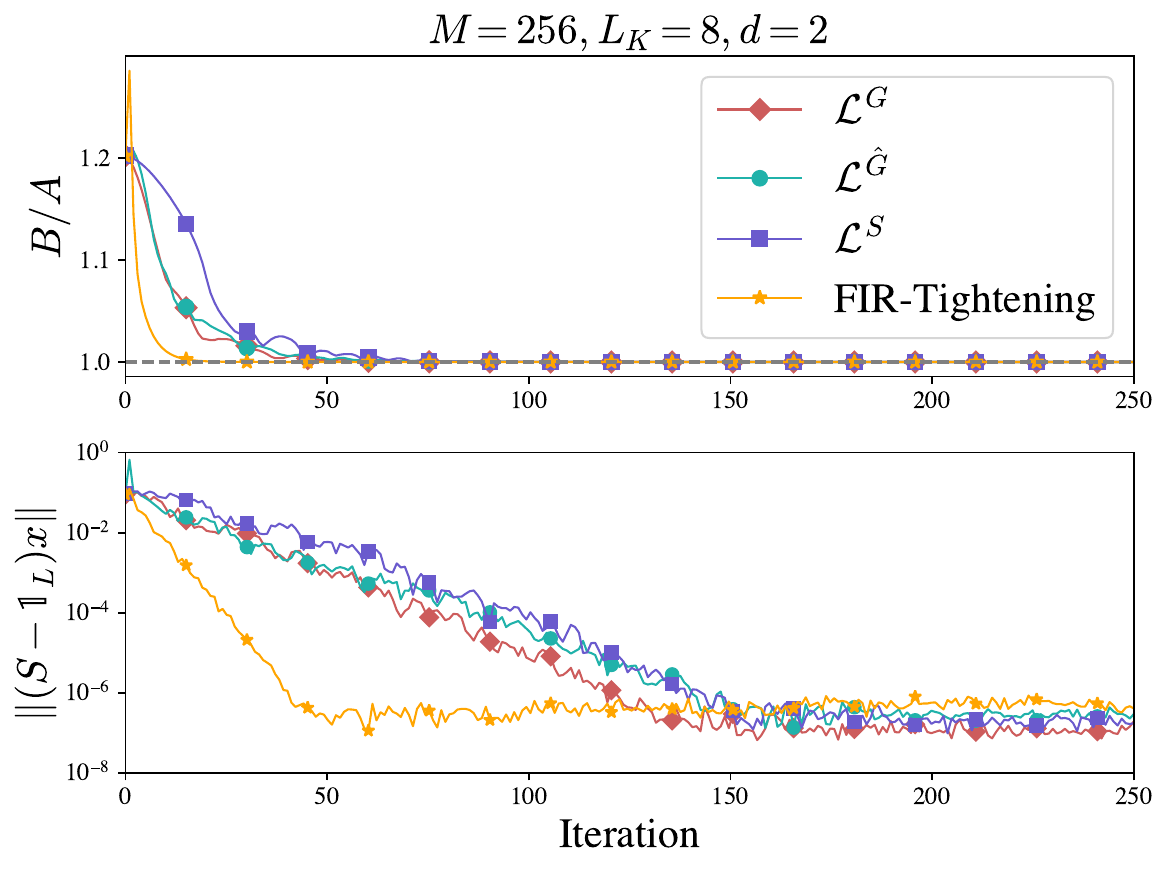}
    \end{subfigure}
    \hfill
    \begin{subfigure}[t]{0.49\textwidth}
        \centering
        \hfill
        \includegraphics[width=\textwidth]{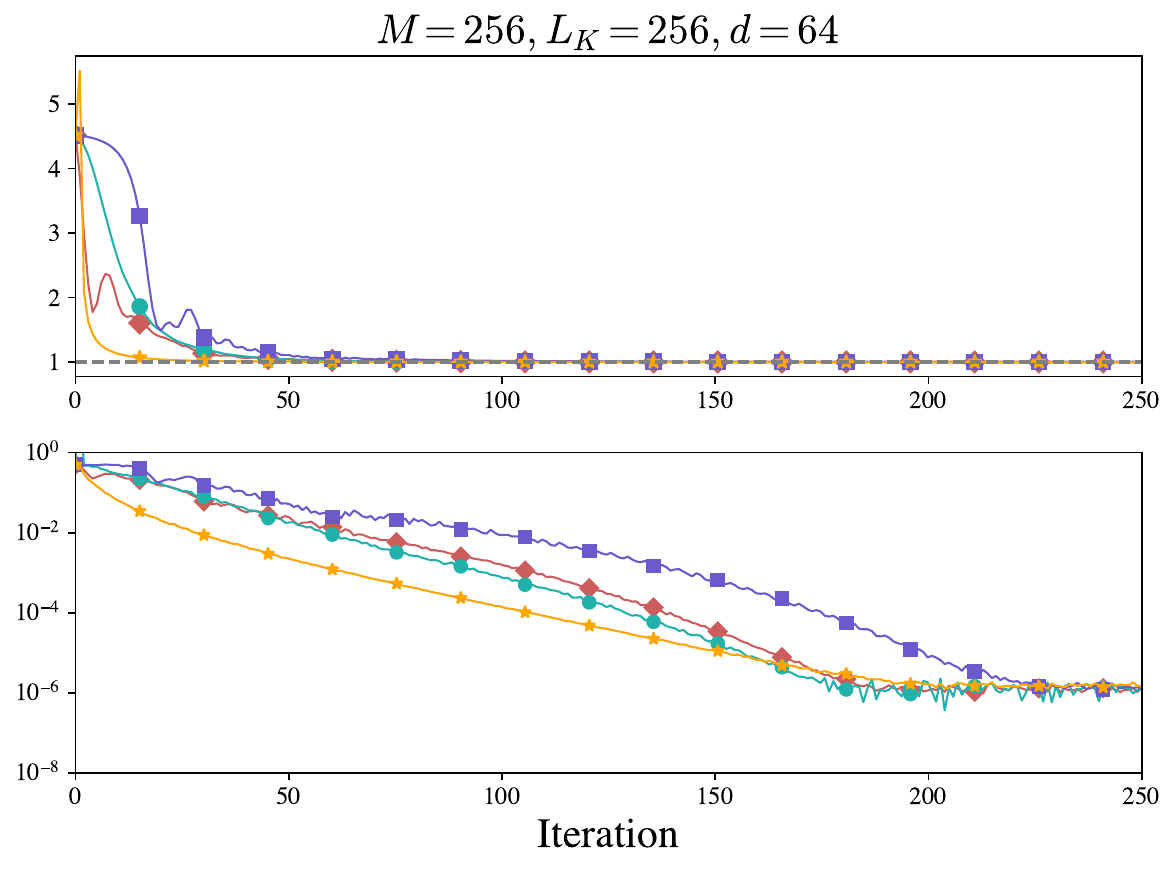}
    \end{subfigure}
    \caption{Progress of optimizing randomly initialized strided convolutional layers to be Parseval stable (250 iterations). We compare the three objectives $\mathcal{L}^G$, $\mathcal{L}^{\hat{G}}$, $\mathcal{L}^S$, and FIR-tightening. Top: Condition number $B/A$. Bottom: Reconstruction error $\Vert (S- I_L) x \Vert$ for random $x$, plotted on a log-scale.
    Left: A common setting for a single-channel layer. Right: The single-channel pendent that corresponds to a multi-channel layer with $C=32$ input channels, $L_K=8$, and $d=2$.
    In general, for any setting where the initialization yields a g-frame, all methods are capable of getting to Parseval stability at machine precision.}
    \label{fig:tightenup}
    \hspace{1cm}
    \includegraphics[width=\linewidth]{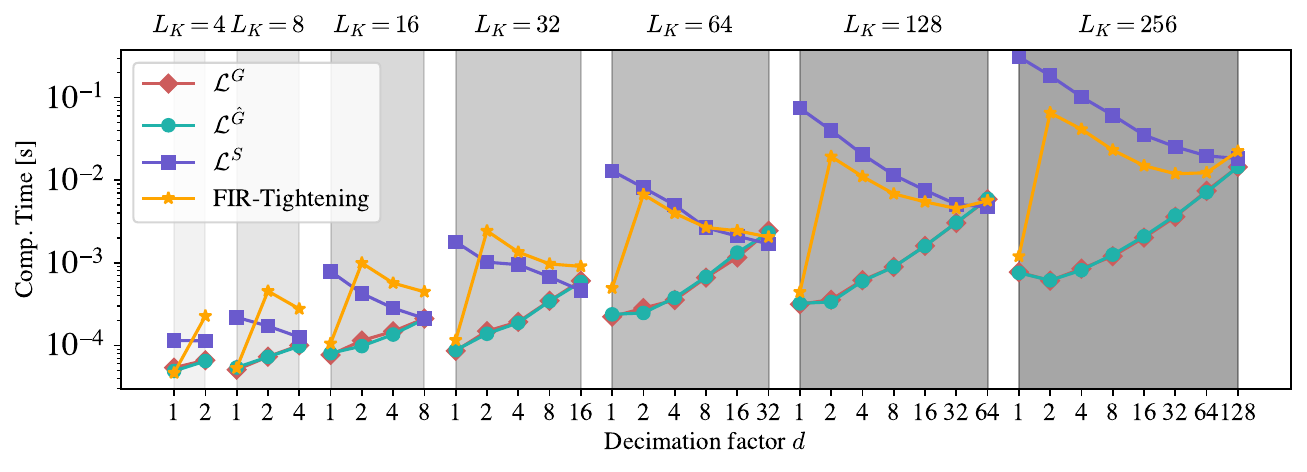}
    \caption{Average computation time for one step (log scale). We compare different kernel sizes $L_K$ and decimation factors $d\leq L_K/2$. We set $M=256$. Except for $d= L_K/2$, the aliasing-based methods $\mathcal{L}^G$ and $\mathcal{L}^{\hat{G}}$ are significantly faster than the baseline $\mathcal{L}^S$ (up to $32$ times) and FIR-tightening.}
    \label{fig:speed}
\end{figure}

\newpage
\begin{definition}[FIR-tightening]
    Let $\{(w_j^{(0)})_{j=1}^M,\downarrow_d\}$ be a filterbank
    that forms a g-frame for $\CC^L$ and $\mathbf{P}_{L_K}:\CC^L\rightarrow \CC^L$ be the projection operator onto the first ${L_K}\leq L$ coordinates. We define the $n$-FIR-tightened filterbank $\{(w_j^{(n)})_{j=1}^M,\downarrow_d\}$ by the filters 
    \begin{equation}\label{eq:firtight}
        w_j^{(n)} = \mathbf{P}_{L_K}\left(\left(S^{(n-1)}\right)^{-\frac{1}{2}}   w_j^{(n-1)}\right),
    \end{equation}
    where $S^{(n-1)}$ is the frame operator for the filterbank $\{(w_j^{(n-1)})_{j=1}^M,\downarrow_d\}$.  
\end{definition}
Clearly, the filters $w_j^{(n-1)}$ have kernels of size $L_K$ at any iteration $n>0$.
However, we can not guarantee exact Parseval stability due to the projection step.
Yet, in many numerical experiments, we observed that FIR-tightening reliably converges to approximately Parseval g-frames in few iterations. It should be noted, however, that this procedure is very invasive due to the projection, and computationally very expansive as it required a matrix inversion in every step.\\

We initialize the kernels as i.i.d.~Gaussian random vectors with zero mean and variance $(L_K M)^{-1}$. We use stochastic gradient decent with learning rates between $10^{-2}$ and $10^{-4}$ and add an adaptive regulation by scaling the objective with $(B/A-1)^{1/10}$ in every iteration.
In Figure \ref{fig:tightenup}, we show the progress of the conditions numbers $B/A$ and the reconstruction errors $\Vert (S- I_L) x \Vert$ for random and normalized $x$ over $250$ iterations.
All methods were capable of
getting to Parseval stability at machine precision ($B/A \approx 1+10^{-7}$ and $\Vert (S- I_L) x \Vert\approx 10^{-7}$) in any considered scenario.

When looking at the computation time\footnote{The computations were done on the CPU of a MacBook Pro (M4, 14-core CPU).} that the evaluations of the regularizers take, the aliasing-based methods outperform in all relevant settings, especially when the stride is small. In Figure \ref{fig:speed} we compare the average speed of computing the different objectives in different settings, and plot the times on a log scale. In very low parameter settings all the methods are approximately equally fast. For any increase in $M$ or $L_K$, the aliasing-based methods take the lead. For the setting $M=256,L_K=64, d=2$ they are approximately $32$ times faster than the baseline.

\section{Aliasing at Random Initialization}\label{sec:random}

Initializing neural network weights with random draws from a probability distribution is a standard practice. In convolutional layers, this means that the kernels are random vectors. The statistical behavior of aliasing at this stage can be used to understand how much a feature representation is affected by aliasing before any learning occurs. This insight can guide the design of initialization schemes to mitigate destructive aliasing and potentially enhance learning dynamics.

\subsection{Expected value and variance}
We assume that the kernels of the filters $w_{j}$ are i.i.d.~random vectors of length $L_K$ and the rest of the $L-L_K$ filter entries are zero. For the Gaussian or uniform i.i.d.~case, the distribution looks like
\begin{equation}\label{eq:rand}
    w_j[n]\sim\begin{cases}
        \mathcal{N}(\mu,\sigma^2) \text{ or }\mathcal{U}_{[a, b]}& \text{if } 0\leq n\leq L_K-1\\
        0& \text{otherwise.}
    \end{cases}
\end{equation}
By the extension with zeros, the $w_j$ are not independent random vectors.
The following theorem shows how the expected value and variance of the aliasing terms of such a filterbank behave.
Note that the signal length $L$ appearing as a factor in all the expression comes from the unitary DFT and would not pop up with another normalization scheme.

\begin{theorem}\label{thm:randfb}
    Let $\{(w_j)_{j=1}^M,\downarrow_d\}$ with kernel size $L_K$ be a filterbank for $\CC^L$, where the $w_j$ are distributed as in~\eqref{eq:rand} with zero mean and variance $\sigma^2$.
    The corresponding aliasing terms $G_n$ are complex random vectors satisfying the following. For any $0\leq k\leq L-1$ the expected value is given by
    \begin{equation}\label{eq:expOG}
        \mathbb{E}\left[ G_n[k] \right] = \frac{\sigma^2M}{Ld}\sum_{\ell=0}^{L_K-1}e^{-2\pi i \ell n /d}.
    \end{equation}
    In particular, if $\tfrac{L_K}{d}\in \NN$, then
\begin{equation}\label{eq:div}
    \mathbb{E}\left[ G_n[k] \right] =
    \begin{cases}
        \frac{\sigma^2ML_K}{Ld}&\text{if }n=0\\
        0&\text{if }1\leq n\leq d-1.
    \end{cases}
\end{equation}
The variance is given by
    \begin{equation}\label{eq:varOG}
        \mathbb{V}\left[ G_n[k] \right] = \frac{\sigma^4M}{L^2d^2}\left( \frac{\sin^2(L_K\pi(\frac{2k}{L}-\frac{n}{d}))}{\sin^2(\pi(\frac{2k}{L}-\frac{n}{d}))} + L_K^2 \right).
    \end{equation}
    Moreover, for any $0\leq k\leq L-1$
    \begin{equation}\label{eq:varuni}
    \frac{\sigma^4ML_K^2}{L^2d^2}\leq \mathbb{V}\left[ G_n[k] \right]\leq \frac{2\sigma^4ML_K^2}{L^2d^2}.
\end{equation}
\end{theorem}

\begin{proof}
    The proof for both expressions goes by expanding the Fourier transforms of $\hat{w}_j$ and $\overline{\mathbf{T}_{nL/d}\hat{w}_j}$ and using the i.i.d.~property of the kernels, 
    \begin{equation}\label{eq:cov}
        \text{Cov}\left(w_j[\ell], w_j[\ell']\right)=\begin{cases}
        \sigma^2\delta[\ell-\ell']& \text{for } 0\leq\ell,\ell'\leq T-1\\
        0& \text{for } T\leq\ell,\ell'\leq L-1.
    \end{cases}
    \end{equation}
    The expression for the mean in \eqref{eq:expOG} is given by
    \begin{align}
        \mathbb{E}\left[G_n[k]\right] 
        &= \frac{1}{Ld}\sum_{j=1}^M \sum_{\ell,\ell'=0}^{L_K-1} e^{-2\pi i \frac{\ell k}{L}} e^{2 \pi i \frac{\ell'(k-nL/d)}{L}} \text{Cov}\left(w_j[\ell], w_j[\ell']\right)\\
        &= \frac{\sigma^2M}{Ld} \sum_{\ell=0}^{L_K-1} e^{-2\pi i \ell n/d }.
    \end{align}
    If $\tfrac{L_K}{d}\in \NN$, then the exponential terms are roots of unity, which shows \eqref{eq:div}.
    
    We turn to the variance. Since the terms $\hat{w}_j[k] \overline{\hat{w}_j[k-n\tfrac{L}{d}]}$ are independent across $j$, we can pull out the sum and get that
    \begin{align}
    \mathbb{V}\left[G_n[k]\right]=&\;
    d^{-2}
    \sum_{j=1}^M
    \mathbb{E}\left[\left|\hat{w}_j[k] \overline{\hat{w}_j[k-n\tfrac{L}{d}]}\right|^2\right]
    -\left|\mathbb{E}\left[\hat{w}_j[k] \overline{\hat{w}_j[k-n\tfrac{L}{d}]}\right]\right|^2.
    \end{align}
    For the first term, we find that
    \begin{align}
        &\mathbb{E}\left[\left|\hat{w}_j[k] \overline{\hat{w}_j[k-n\tfrac{L}{d}]}\right|^2\right] \\
        & =\frac{1}{L^2} \sum_{m,m',\ell,\ell'=0}^{L_K-1} e^{-2\pi i \frac{k(m-m')}{L}} e^{-2 \pi i \frac{(k-nL/d)(\ell-\ell') }{L}} \mathbb{E}\left[w_j[m] w_j[\ell]w_{j}[m'] w_{j}[\ell']\right].\label{eq:var1st}
    \end{align}
    Applying Isserlis' Theorem~\cite{isserlis1918formula} and using \eqref{eq:cov} gives
    \begin{align}
        \mathbb{E}\left[w_j[m] w_j[\ell]w_{j}[m'] w_{j}[\ell']\right]= &\;
        \mathbb{E}\left[w_j[m] w_j[\ell]\right]\cdot \mathbb{E}\left[w_{j}[m'] w_{j}[\ell']\right] \\
        & + \mathbb{E}\left[w_j[m] w_{j}[m']\right]\cdot \mathbb{E}\left[w_j[\ell] w_{j}[\ell']\right]\\
        & + \mathbb{E}\left[w_j[m] w_{j}[\ell']\right]\cdot \mathbb{E}\left[w_{j}[m'] w_j[\ell]\right]
    \end{align}
    \vspace{-2em}
    \begin{align}
        =\sigma^4\big(\delta[m-\ell]\cdot \delta[m'-\ell']
        +\delta[m-m']\cdot \delta[\ell-\ell']
        + \delta[m-\ell']\cdot \delta[m'-\ell]\big).
    \end{align}
    This reduces the quadruple sum in \eqref{eq:var1st} to double sums for three different configurations of $m,\ell, m',\ell'$, determined by the covariances. In the first case ($m=\ell, m'=\ell'$), the sums simplify to $\sum_{m,m'=0}^{L_K-1} e^{-2 \pi i (m-m')(\frac{2k}{L}-\frac{n}{d}) }$. In the second case ($m=m', \ell=\ell'$), all exponents cancel, hence, the sums add up to $L_K^2$. The third case ($m=\ell', m'=\ell$), coincides with the second term in the variance,
    \begin{align}
        &\left|\mathbb{E}\left[\hat{w}_j[k] \overline{\hat{w}_j[k-n\tfrac{L}{d}]}\right]\right|^2
        = \frac{\sigma^4}{L^2} \sum_{\ell,\ell'=0}^{L_K-1} e^{2\pi i (\ell-\ell') n/d }.
    \end{align}
    Hence, the two terms cancels.
    In total, this gives us
    \begin{align}
        \mathbb{V}\left[ G_n[k] \right] = \frac{\sigma^4M}{L^2d^2} \left(  \sum_{m,m'=0}^{L_K-1} e^{-2 \pi i (m-m')(\frac{2k}{L}-\frac{n}{d}) } + L_K^2
        \right).
    \end{align}
    We can further simplify the expression by setting $\tau = m-m'$ and recognizing the Fourier transform of a triangular function at $\frac{2k}{L}-\frac{n}{d}$ which has a known closed form expression in terms of $\sin^2$~\cite{mallat08wavelettour}. We get
    \begin{align}
        \mathbb{V}\left[ G_n[k] \right]&= \frac{\sigma^4M}{L^2d^2} \left( \sum_{\tau=-L_K+1}^{L_K-1}(L_K-\vert \tau\vert) e^{-2 \pi i \tau (\frac{2k}{L}-\frac{n}{d}) } + L_K^2\right)\\
        &=\frac{\sigma^4M}{L^2d^2} \left( \frac{\sin^2(L_K\pi(\frac{2k}{L}-\frac{n}{d}))}{\sin^2(\pi(\frac{2k}{L}-\frac{n}{d}))} + L_K^2\right).
    \end{align}
    Finally, the fact that $0\leq \tfrac{\sin^2(L_K\pi t)}{\sin^2(\pi t)}\leq L_K^2$ for all $t\in \mathbb{R}$ shows~\eqref{eq:varuni}.
\end{proof}

From the expression of the variance we see that the aliasing terms peak for $\frac{2k}{L}-\frac{n}{d}\in \ZZ$. Given that $k$ takes values from $0$ to $L-1$ and $0\leq n \leq d-1$, for every $G_n$ this happens exactly twice, namely at $k=\frac{nL}{2d}$ and $k=\frac{nL}{2d}+\frac{L}{2}$. Figure \ref{fig:stats} shows empirical statistics of the $G_n$ for a random $d$-decimated filterbank, together with the theoretical expressions from Theorem \ref{thm:randfb}. We see a behavior that looks like a uniform distribution with mean at one that peaks around the two mentioned values. Note that the peaks comes from the fact that the filters arise by applying rectangular windows of size of the kernels which accumulate energy at their center frequency at zero.\\

In practice, setting the variance for initializing the kernels in dependence of the total number of parameters in the layer is a common way to balance the energy of the layer. According to our findings in Thm.~\ref{thm:randfb}, we should set $\sigma^2 = \frac{Ld}{ML_K}$.

\begin{corollary}
    Let $\tfrac{L_K}{d}\in \NN$ and $\sigma^2 = \frac{Ld}{ML_K}$. Then
    \begin{equation}
        \mathbb{E}[ \hat{S} ]=\mathbb{E}[ S ]= I_L
    \end{equation}
    and
    \begin{equation}\label{eq:varr}
        M\inv\leq \mathbb{V}\left[ G_n \right]\leq 2 M\inv.
    \end{equation}
\end{corollary}
This makes precise that the stability of a random filterbank, distributed with a variance that is inversely proportional to the number of parameters, depends merely on the number of channels. Given that the variance bounds in \eqref{eq:varr} are very coarse by only considering the two peaks, we find that a random filterbank (with appropriate scaling) seems to behave well in reasonable settings. Even for only $M=8$ channels, the variance lies between $0.125$ and $0.25$. For $M=256$, the bounds already become $0.004$ and $0.008$.

\begin{figure}
    \centering
    \includegraphics[width=1\linewidth]{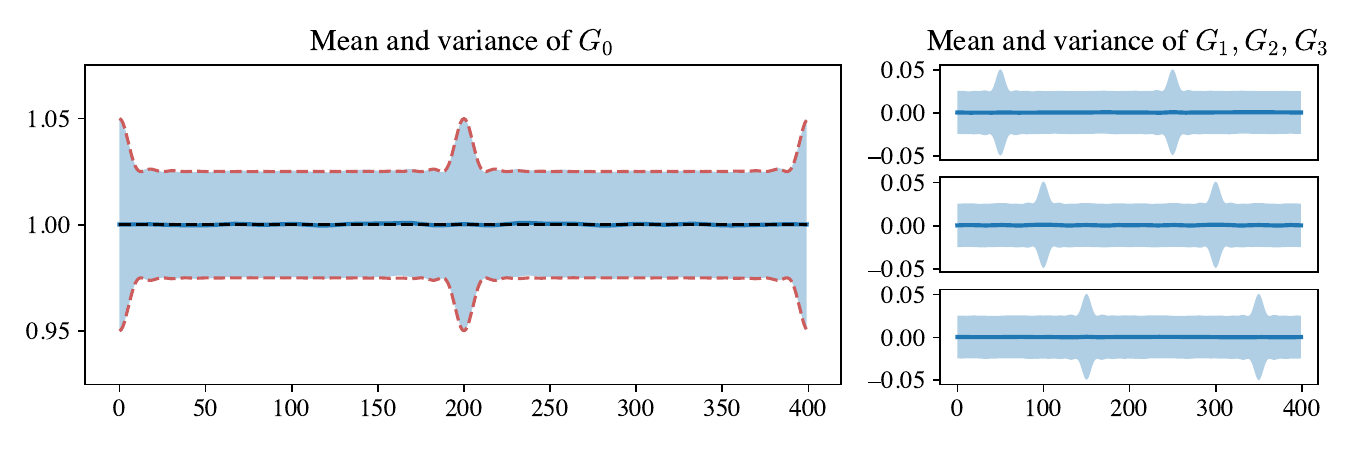}
    \caption{Empirical and theoretical expected values and variances of the aliasing terms of a filterbank with $M=40$ random filters in $\CC^{400}$, kernel size $L_K=16$, and stride $d=4$. The empirical statistics are computed from drawing the random filterbank $100$k times. The dashed lines represent the theoretical expressions from Thm.~\ref{thm:randfb}.}
    \label{fig:stats}
\end{figure}

\section{Generalizations}\label{sec:multi}

In this last section, we describe how more general settings can be converted to the filterbank format $\{(w_j)_{j=1}^M, \downarrow_{d}\}$ that we assumed throughout the paper. With these adaption, any result from the previous sections can be applied directly.

\subsection{Multi-channel layers}

Convolutional layers are commonly used to process multiple input channels simultaneously. Let $\{(w_{1,j})_{j=1}^{M},\dots, (w_{C,j})_{j=1}^{M}, \downarrow_{d}\}$ denote a multi-channel filterbank for $\CC^L$ with $C$ input channels and $M$ output channels, where $w_{i,j}\in \CC^L$ denotes the filter associated with input channel $i$ and output channel $j$. Let $x_i$ denote the input from channel $i$ then the transform (analysis operator) for the multi-channel filterbank $\mathbf{\Theta}:\CC^{CL}\rightarrow \CC^{ML/d}$ is given by
\begin{equation}
    (x_1,\ldots,x_C)\mapsto \left( \sum_{i=1}^C (x_i \ast w_{i,j})\hspace{-0.25em}\downarrow_d\right)_{j=1}^M.
\end{equation}
By interlacing the filters across all channels, we obtain a filterbank with equivalent stability properties.

\begin{lemma}
    The frame bounds of the filterbank $\{(w_{1,j})_{j=1}^{M},\dots, (w_{C,j})_{j=1}^{M}, \downarrow_{d}\}$ in $\CC^L$ are given by the ones of the filterbank
    \begin{equation}
        \{(\mathrm{\mathbf{w}}_j)_{j=1}^M,\downarrow_{Cd}\}\text{ in }\CC^{CL},
    \end{equation}
    with filters $\mathrm{\mathbf{w}}_j\in \CC^{CL}$ given by
    \begin{equation}
        \mathrm{\mathbf{w}}_j[nC+i-1]=w_{i,j}[n].
    \end{equation}
    The kernel size is $CL_K$.
\end{lemma}

\begin{proof}
    Let $\mathrm{\mathbf{x}} \in \CC^{CL}$ be given by $\mathrm{\mathbf{x}}[nC+i-1]=x_{i}[n]$ 
    then a straightforward calculation shows that for every $j$,
    \begin{equation}
    \sum_{i=1}^C (x_i \ast w_{i,j})\hspace{-0.25em}\downarrow_d\; = (\mathrm{\mathbf{x}}\ast \mathrm{\mathbf{w}}_j)\hspace{-0.25em}\downarrow_{Cd}.
    \end{equation}
    The claim follows since interlacing is a unitary operation.
\end{proof}

The corresponding g-frame has elements $\mathrm{\mathbf{T}}_j:\CC^{CL}\rightarrow \CC^{L/d}$ given by $x\mapsto (x\ast \mathrm{\mathbf{w}}_j)\hspace{-0.25em}\downarrow_{Cd}$
and the frame operator is given by $\mathrm{\mathbf{S}} =\mathbf{\Theta}^*\mathbf{\Theta}= \sum_{j=1}^M \mathrm{\mathbf{T}}_j^* \mathrm{\mathbf{T}}_j$. Tightness means $\mathrm{\mathbf{S}} = A\cdot  I_{CL}$.\\

In the multichannel setting, there are situations where we cannot (and also do not want to) ensure the g-frame property on $\CC^{CL}$ due to specific architecture designs or dimensionality constraints, i.e., $C > M/d$. In this case, the largest and smallest non-zero eigenvalues of $\mathrm{\mathbf{S}}$ provide the stability bounds on the range of $\mathbf{\Theta}$. We can express this situation conveniently through the frame bounds of a filterbank with different filters.

\begin{lemma}
    Let $\{(w_{i,j})_{i=1,j=1}^{C,M},\downarrow_{d}\}$ be a multi-channel filterbank in $\CC^L$ with $C > M/d$. Then the largest and smallest non-zero eigenvalues of $\mathrm{\mathbf{S}}$ are equal to the frame bounds of the filterbank
    \begin{equation}
        \{(\mathbf{v}_j)_{j=1}^d, \downarrow_M\}\text{ in }\CC^{ML/d},
    \end{equation}
    where the filters $\mathbf{v}_j\in \CC^{ML/d}$ are given by
    \begin{equation}
        \mathbf{v}_j[nM+k-1]=\overline{\mathbf{w}_k[nCd+j-1]}.
    \end{equation}
\end{lemma}

\begin{proof}
    Through the described row and column rearrangements one obtains that the frame operator for the filterbank $\{(\mathbf{v}_j)_{j=1}^d, \downarrow_M\}$ is unitarily equivalent to the Gram matrix $\mathrm{\mathbf{G}}= \mathbf{\Theta}\mathbf{\Theta}^*$ for the filterbank $\{(\mathrm{\mathbf{w}}_j)_{j=1}^M,\downarrow_{Cd}\}$.
    Since Gram matrix and frame operator share all non-zero eigenvalues, the claim follows.
\end{proof}

Tightness in the sense of $\mathrm{\mathbf{\tilde{S}}} = A\cdot  I_{ML/d}$ now means that all non-zero eigenvalues of $\mathrm{\mathbf{S}}$ are equal to $A$. We can interpret this as the g-frame for the original filterbank being tight on the subspace $\CC^{ML/d}$. For $A=1$, this is equivalent to the row vectors of the matrix representation of $\mathbf{\Theta}$ being orthonormal~\cite{massart2022rectorth}.

\subsection{Non-uniform stride}

By default, all convolutions in a strided convolutional layer use the same (uniform) stride across the channels. Yet, if a layer architectures requires a non-uniform stride configuration (as e.g., in~\cite{lostanlen2023murenn}),
we can instead consider an equivalent uniform filterbank that contains shifted copies of the filters.
\begin{lemma}
    Let $\{(w_j)_{j=1}^M, \downarrow_{d_j}\}$ be a non-uniform filterbank in $\CC^L$ with channel-specific strides $d_1,\dots,d_M\geq1$. Then the frame bounds are given by the ones of the filterbank
    $$\{((\mathbf{T}_{nd_j}w_j)_{n=0}^{D_j})_{j=1}^{M}, \downarrow_{\mathbf{d}}\}\text{ in }\CC^L,$$
    where $\mathbf{d}=\operatorname{lcm}(d_1,\dots,d_M)$ and $D_j=\mathbf{d}/d_j$.
\end{lemma}

\begin{proof}
    The statement follows from the $\ell^2(\mathbb{Z})$ case for which a proof can be found in~\cite{balazs2017framespsycho}.
\end{proof}

\subsection{Dilated convolution}
For $a\geq 1$, dilated convolution of $x$ and $w$ in $\CC^L$ is defined by
\begin{equation}
    (x \ast_a w)[n] = \sum_{\ell=0}^{L-1}x [a\ell]\overline{w[n-\ell]}.
\end{equation}
This is equivalent to dilate the kernel of the filter by $a$ directly, which can be realized by upsampling the filter and cropping the result back to length $L$. We denote the cropping by $w_{:L}$.
\begin{lemma}
    Let $\{(w_j)_{j=1}^M,\downarrow_d, \wedge_a\}$ denote a filterbank in $\CC^L$ that applies dilated convolution with $a\leq L/L_K$. The frame bounds are given by the ones of the filterbank
    \begin{equation}
        \{((w_j\hspace{-0.25em}\uparrow_a)_{:L})_{j=1}^{M}, \downarrow_{d}\}\text{ in }\CC^L,
    \end{equation}
    with kernel sizes $a L_K$.
\end{lemma}

\begin{proof}
    It is easy to see that
    \begin{equation}\label{eq:dil}
        (x \ast_a w)\hspace{-0.25em}\downarrow_d[n] = \sum_{\ell=0}^{L-1}x [a\ell]\overline{w[dn-\ell]} = (x \ast (w\hspace{-0.25em}\uparrow_a)_{:L})\hspace{-0.25em}\downarrow_d[n].
    \end{equation}
    Assuming $a\leq L/L_K$ guarantees that we only crop zeros. Hence, the frame bounds are not affected.
\end{proof}

\section{Discussion and Outlook}\label{sec:conclusion}

This works shows that generalized frames offer a natural formalism to study aliasing in strided convolutional layers and link it to their invertibility and stability.
We can elegantly formulate perfect reconstruction conditions, which only depend on the kernel size and the stride, and get tightness characterization with an extension to the painless case via sub-band restriction. With this we can design computationally efficient optimization objectives that are nicely interpretable by means of aliasing suppression. Using the aliasing-based objectives as regularizers to promote Parseval stability in a larger-scale application did not fit into the scope of this work and remains to be tested. In the same context, it is interesting to ask how stability is affected by a learning update in general, and how regularization can intrinsically bias the updates towards a local minimum which yields at least a comparable performance but with stability guarantees.
By a standard frame perturbation result (see e.g.~\cite{waldron2018finitetight}), the worst-case scenario is determined by the upper frame bound of the update filterbank with filters $u_j$ given by $u_j[n]=\Phi\left((w_j)_{j=1}^M\right)[j,n]$, where $\Phi$ represents the update rule (e.g., the gradients).
\begin{lemma}\label{lem:gd}
    Let $\{(w_j)_{j=1}^M,\downarrow_d\}$ be a filterbank with frame bounds $A,B$ and $R$ the optimal upper frame bound for $\{(u_j)_{j=1}^M,\downarrow_d\}$. If $\gamma \cdot R<A$ then $\{(w_j-\gamma\cdot u_j )_{j=1}^M,\downarrow_d\}$ is a g-frame with frame bounds $(\sqrt{A}-\gamma\sqrt{R})^2,(\sqrt{B}+\gamma\sqrt{R})^2$.
\end{lemma}
While this means that after each weight update, the frame bounds might worsen by $\gamma \sqrt{R}$, an interesting question for future research is if and how regularization can avoid these worst-case scenarios.\\

In the context of random initialization, this work provides the fundamental statistical properties of the aliasing terms, but leaves some open questions on how to apply them. First, how can we obtain the statistics for the frame bounds from the statistics of the aliasing terms? While it seems to be a natural application, we can not use norm properties of random vectors without knowing their distribution. Such an analysis has to be done in a more comprehensive follow-up work that focuses on randomness. Second, we observed that the peaky behavior of the variance around two entries comes from the accumulation of energy at the center frequency of the rectangular windows that we apply. An approach to better distribute the peaks in the variance could be to apply smooth window functions with different frequencies.

\section*{Acknowledgment}
D. Haider is a recipient of a DOC Fellowship (A 26355) of the Austrian Academy of Sciences at the Acoustics Research Institute. The work of V. Lostanlen was supported by the ANR project MuReNN (ANR-23-CE23-0007-01). P. Balazs was supported by the FWF projects LoFT (P 34624), NoMASP (P 34922), Voice Prints (P 36446), and the WWTF project EleCom (LS23-024). The work of N. Holighaus was supported by the FWF project DISCO (PAT4780023).

\newpage
\bibliographystyle{abbrv}
\bibliography{references}

\end{document}